\documentclass{article}

\usepackage{arxiv}

\usepackage[utf8]{inputenc} % allow utf-8 input
\usepackage[T1]{fontenc}    % use 8-bit T1 fonts
\usepackage{hyperref}       % hyperlinks (must be loaded for cross-referencing)
\hypersetup{
    colorlinks=true,   % Removes boxes and colors the text instead
    linkcolor=blue,    % Color for internal links (equations, figures, etc.)
    filecolor=blue,    % Color for local file links
    urlcolor=blue,     % Color for linked URLs
    citecolor=blue     % Color for bibliography citations
}
\usepackage{url}            % simple URL typesetting
\usepackage{booktabs}       % professional-quality tables
\usepackage{amsfonts}       % blackboard math symbols
\usepackage{nicefrac}       % compact symbols for 1/2, etc.
\usepackage{microtype}      % microtypography
\usepackage{graphicx}
\usepackage[numbers,sort&compress]{natbib}
\graphicspath{ {./images/} }
\usepackage{amsmath}
\usepackage{amsthm}
\usepackage{colortbl}
\usepackage{xcolor}
\usepackage{multirow}

\definecolor{first}{rgb}{1.0, 0.85, 0.85}
\definecolor{second}{rgb}{0.85, 1.0, 0.85}
\definecolor{third}{rgb}{0.85, 0.9, 1.0}  

\newtheorem{theorem}{Theorem}
\newtheorem{lemma}{Lemma}
\newtheorem{definition}{Definition}

\title{Beyond the Laplacian: Doubly Stochastic Matrices for Graph Neural Networks}

\renewcommand{\today}{}

\author{
 Zhaobo HU \\
 SAMOVAR, T\'el\'ecom SudParis \\
 Institut Polytechnique de Paris\\
 \texttt{zhaobo.hu@telecom-sudparis.eu} \\
 \And
 Vincent Gauthier \\
 SAMOVAR, T\'el\'ecom SudParis \\
 Institut Polytechnique de Paris\\
 \texttt{vincent.gauthier@telecom-sudparis.eu} \\
 \And
 Mehdi Naima \\
 CNRS -- LIP6\\
 Sorbonne Université\\
 \texttt{mehdi.naima@lip6.com} \\
}

\begin{document}
\maketitle

% \begin{abstract}
% Graph Neural Networks (GNNs) typically rely on the standard Laplacian or adjacency matrix for message passing. This paper proposes replacing the traditional Laplacian with a Doubly Stochastic graph Matrix (DSM) derived from the inverse of the modified Laplacian. The DSM natively captures continuous topological proximity and acts as a strict local centrality indicator. To circumvent the computationally prohibitive $O(n^3)$ exact matrix inversion, we introduce DsmNet, which employs an efficient truncated Neumann series approximation. Furthermore, because naive truncation results in probability mass leakage, we develop DsmNet-compensate, incorporating a Residual Mass Compensation mechanism that analytically folds residual mass back into self-loops, strictly preserving row-stochasticity. Theoretical analysis and extensive empirical evaluations demonstrate that our decoupled architectures scale in $O(K|E|)$ time. The proposed models achieve highly competitive performance across homophilic datasets and serve as effective structural encodings in generalized Graph Transformers, while we also rigorously analyze their performance boundaries on heterophilic graphs.
% \end{abstract}

\begin{abstract}
Graph Neural Networks (GNNs) conventionally rely on standard Laplacian or adjacency matrices for structural message passing. In this work, we substitute the traditional Laplacian with a Doubly Stochastic graph Matrix (DSM), derived from the inverse of the modified Laplacian, to naturally encode continuous multi-hop proximity and strict local centrality. To overcome the intractable $O(n^3)$ complexity of exact matrix inversion, we first utilize a truncated Neumann series to scalably approximate the DSM, which serves as the foundation for our proposed DsmNet. Furthermore, because algebraic truncation inherently causes probability mass leakage, we introduce DsmNet-compensate. This variant features a mathematically rigorous Residual Mass Compensation mechanism that analytically re-injects the truncated tail mass into self-loops, strictly restoring row-stochasticity and structural dominance. Extensive theoretical and empirical analyses demonstrate that our decoupled architectures operate efficiently in $O(K|E|)$ time and effectively mitigate over-smoothing by bounding Dirichlet energy decay, providing robust empirical validation on homophilic benchmarks. Finally, we establish the theoretical boundaries of the DSM on heterophilic topologies and demonstrate its versatility as a continuous structural encoding for Graph Transformers.
\end{abstract}

% ----------------------------------------------------------------------
% Section 1: Introduction
% ----------------------------------------------------------------------
\section{Introduction}
\label{sec:intro}
Graph Neural Networks (GNNs) have emerged as the primary paradigm for representation learning on non-Euclidean data \citep{kipf2017semi, hamilton2017inductive, velickovic2018graph}. The core mechanism of these models involves iteratively aggregating feature information from local node neighborhoods using structural operators based on the graph topology. While highly effective, traditional discrete local operators often struggle to continuously quantify multi-hop topological dependencies without uniformly mixing node features. To systematically capture more complex structural relationships, we propose replacing the standard Laplacian with the Doubly Stochastic graph Matrix (DSM). Derived from the inverse of the modified Laplacian ($\tilde{L} = I + L$), the DSM offers mathematically superior properties for information diffusion \citep{merris1997doubly, andrade2025doubly}. Unlike standard matrices, the DSM provides a continuous measure of topological proximity, capturing multi-hop relationships while naturally decaying with distance. Furthermore, its diagonal entries serve as a robust geometric indicator of node centrality, embedding critical structural priors directly into the diffusion operator. Despite these advantages, computing the exact DSM requires a dense matrix inversion with $O(n^3)$ complexity, rendering it intractable for large-scale graphs. To overcome this, we decouple the neural transformation from the topological diffusion and introduce DsmNet, which approximates the DSM using an efficient truncated Neumann series. However, truncating an infinite series inherently disrupts the closed algebraic system, resulting in probability mass leakage and the loss of row-stochasticity. To correct this structural bias, we further propose DsmNet-compensate, which features a mathematically rigorous Residual Mass Compensation mechanism. This mechanism analytically re-injects the truncated tail mass back into the ego-node's self-loop, strictly preserving mass conservation and bounding the receptive field without sacrificing topological fidelity. The main contributions of this paper are summarized as follows:
\begin{itemize}
    \item We introduce the Doubly Stochastic graph Matrix (DSM) as a replacement for the standard Laplacian in GNNs, demonstrating its theoretical advantages in modeling continuous proximity and local centrality.
    \item We formulate DsmNet, utilizing a scalable $O(K|E|)$ truncated Neumann series approximation for the DSM.
    \item We develop DsmNet-compensate, introducing a novel Residual Mass Compensation mechanism that strictly restores row-stochasticity and structural dominance.
    \item We conduct extensive empirical evaluations across diverse graph topologies. 
\end{itemize}

% ----------------------------------------------------------------------
% Section 2: Related Works
% ----------------------------------------------------------------------
\section{Related Works}
\label{sec:related_works}

\subsection{Graph Neural Networks}
Graph Neural Networks operate by passing messages along the edges of a graph, enabling the learning of node and graph-level representations. The evolution of these networks can be broadly categorized into spectral and spatial approaches. Spectral methods formulate graph convolutions utilizing graph signal processing, applying filters defined in the graph Fourier domain. Spatial methods directly aggregate neighborhood features, introducing inductive capabilities through sampling and dynamic weight assignment via attention mechanisms. While early architectures tightly couple feature transformation with neighborhood aggregation, later paradigms explore decoupled approaches to separate these operations. This decoupling expands the effective receptive field without proportionally increasing model parameters or causing severe gradient vanishing. Our method aligns with this decoupled paradigm, replacing heuristic propagation matrices with a mathematically derived, mass-conserving inverse Laplacian approximation.

\subsection{Doubly Stochastic Matrix and Structural Optimization}
The mathematical foundations of the Doubly Stochastic graph Matrix originate from algebraic graph theory. The properties of the modified Laplacian inverse ($\tilde{L}^{-1}$) and its capacity to measure structural resistance and node centrality have been rigorously analyzed in matrix theory \cite{merris1997doubly, andrade2025doubly}. These studies mathematically prove that the diagonal entries of the DSM are strictly dominant and inversely proportional to node centrality, providing a robust theoretical basis for utilizing the DSM as a diffusion operator. In parallel, structural optimization and residual compensation techniques have demonstrated critical importance in scaling deep representation learning. In the domain of Large Language Models (LLMs), architectures such as Manifold-Constrained Hyper-Connections (mHC) \cite{xie2025mhc} address training instability caused by diversified connectivity patterns. By projecting expanded residual spaces back onto a specific manifold, mHC restores the identity mapping property intrinsic to residual connections, thereby ensuring efficiency and stability. Our Residual Mass Compensation mechanism applies an analogous theoretical principle to graph topologies: by identifying the algebraic leakage caused by truncation and analytically folding it back into the diagonal, we constrain the approximated operator to the exact row-stochastic manifold, restoring structural integrity and diffusion stability.

% ----------------------------------------------------------------------
% Section 3: Approximating the Doubly Stochastic Graph Matrix
% ----------------------------------------------------------------------
\section{Approximating the Doubly Stochastic Graph Matrix}
\label{sec:dsm_approx}

The modified Laplacian matrix of a graph is defined as $\tilde{L} = I + L$, where $I$ is the identity matrix and $L$ is the standard graph Laplacian. Recent studies show that the inverse of this modified Laplacian, denoted as $B = \tilde{L}^{-1}$, is a doubly stochastic matrix \cite{merris1997doubly, andrade2025doubly}. This doubly stochastic graph matrix (DSM) captures deep structural properties of the graph topology. Specifically, it has the following special properties:

\begin{itemize}
    \item \textbf{Row Attenuation and Proximity:} The entries in the DSM represent a measure of relative proximity between nodes \cite{andrade2025doubly}. In any row $i$, the diagonal entry $b_{ii}$ is strictly the largest element \cite{merris1997doubly, andrade2025doubly}. As we move away from node $i$, the values $b_{ij}$ drop off and decay as a function of the combinatorial distance between node $i$ and node $j$ \cite{andrade2025doubly}. 
    \item \textbf{Node Centrality Indicator:} The diagonal entries of the DSM act as a powerful metric for node centrality \cite{andrade2025doubly}. A smaller value of $b_{ii}$ on the diagonal strongly suggests that node $i$ acts as a \textbf{central} vertex in the graph because it has short distances to many other nodes in the network \cite{andrade2025doubly}.
\end{itemize}

Despite these excellent properties, computing the exact inverse of $\tilde{L}$ requires $O(n^3)$ operations, where $n$ is the number of nodes. For large-scale graphs commonly used in data mining and deep learning, this direct matrix inversion is computationally impossible and requires too much memory. To solve this problem, we propose an efficient approximation method based on matrix splitting and the Neumann Series  \cite{von1947numerical}.

\begin{theorem}[Neumann Series] 
\label{thm:neumann}
Let $P$ be a square matrix. If the spectral radius of $P$ is strictly less than 1, then the matrix $(I - P)$ is invertible, and its inverse can be expanded as the infinite sum:
\begin{equation}
(I - P)^{-1} = \sum_{k=0}^{\infty} P^k
\label{eq:neumann}
\end{equation}
\end{theorem}

We cannot apply Theorem \ref{thm:neumann} directly to the standard Laplacian $L$ because its maximum eigenvalue is typically much larger than 1, which would cause the series to diverge. Instead, we separate the diagonal and off-diagonal parts of the modified Laplacian. Let $D$ be the degree matrix and $A$ be the adjacency matrix, so that $L = D - A$. We can write the modified Laplacian as:
\begin{equation}
\tilde{L} = I + D - A
\label{eq:mod_lap}
\end{equation}

To keep the mathematical notation simple, we define the augmented degree matrix as $\tilde{D} = I + D$. We then factor out $\tilde{D}$ from the equation:
\begin{equation}
\tilde{L} = \tilde{D} - A = \tilde{D} (I - \tilde{D}^{-1}A)
\label{eq:mod_lap_factor}
\end{equation}

Now, we take the inverse of both sides to find the DSM:
\begin{equation}
\tilde{L}^{-1} = (I - \tilde{D}^{-1}A)^{-1} \tilde{D}^{-1}
\label{eq:inverse_split}
\end{equation}

Let us define the transition matrix as $P = \tilde{D}^{-1}A$. The sum of the $i$-th row of $P$ is exactly $d_i / (1 + d_i)$, where $d_i$ is the degree of node $i$. Since $d_i \ge 0$, we have $d_i / (1 + d_i) < 1$. This guarantees that the infinity norm satisfies $\|P\|_{\infty} < 1$. Because the spectral radius $\rho(P)$ of any matrix is bounded by its induced norms, this mathematically ensures that $\rho(P) \le \|P\|_{\infty} < 1$. 

Because the convergence condition $\rho(P) < 1$ is strictly met, we can safely apply Theorem \ref{thm:neumann} to Equation \ref{eq:inverse_split}. The final approximation for the doubly stochastic graph matrix is:
\begin{equation}
\tilde{L}^{-1} = \left( \sum_{k=0}^{\infty} (\tilde{D}^{-1}A)^k \right) \tilde{D}^{-1}
\label{eq:final_approx}
\end{equation}

By truncating this infinite sum to a finite number of steps $K$, we can compute the DSM using only sparse matrix multiplications. This completely avoids the dense matrix inversion and drastically reduces the computational complexity for large graphs.

% ----------------------------------------------------------------------
% Section 4: Spectral Properties
% ----------------------------------------------------------------------
\section{Spectral Properties}
\label{sec:spectral_properties}

The spectral properties \cite{chung1997spectral, spielman2007spectral} of the doubly stochastic graph matrix provide a rigorous mathematical framework for analyzing information propagation on graphs. To understand the diffusion dynamics governed by this matrix, it is necessary to examine its spectrum and establish the precise algebraic relationship between the eigenvalues of the standard Laplacian and those of the doubly stochastic graph matrix.

\begin{lemma}
\label{lem:eigen_b}
Let $L$ be the Laplacian matrix of a graph $G$ with eigenvalues $0 = \lambda_1 \le \lambda_2 \le \dots \le \lambda_n$. Let $B = (I + L)^{-1}$ be the corresponding doubly stochastic graph matrix. The eigenvalues of $B$, denoted as $\mu_i$, are given by $\mu_i = \frac{1}{1+\lambda_i}$, and they are strictly bounded in the interval $(0, 1]$.
\end{lemma}
\textit{Proof.} See Appendix \ref{app:proof_lemma1}.

\begin{definition}
\label{def:spectral_gap}
For a Markov transition matrix or a doubly stochastic operator $P$ with eigenvalues ordered by magnitude $1 = |\mu_1| \ge |\mu_2| \ge \dots \ge |\mu_n|$, the spectral gap $\gamma$ is defined as the difference between the moduli of the two largest eigenvalues, specifically $\gamma = 1 - |\mu_2|$.
\end{definition}

A common source of confusion in algebraic graph theory is whether the spectral gap should be measured using the smallest or the largest eigenvalues. This distinction depends entirely on the nature of the operator being analyzed. The standard graph Laplacian acts as a discrete differential operator. In this continuous-time diffusion context, the system's ground state corresponds to the null space, where the minimal eigenvalue is $0$. The rate at which diffusion minimizes differences across the graph is dictated by the first non-zero eigenvalue, $\lambda_2$, making the gap $\lambda_2 - 0 = \lambda_2$.

\begin{theorem}
\label{thm:spectral_gap}
The spectral gap of the doubly stochastic graph matrix $B$ is determined by the difference between its two largest eigenvalues, giving $\gamma = 1 - \frac{1}{1+\lambda_2}$, where $\lambda_2$ is the algebraic connectivity of $G$.
\end{theorem}
\textit{Proof.} See Appendix \ref{app:proof_theorem2}.

\begin{figure*}[htbp]
    \includegraphics[width=\textwidth]{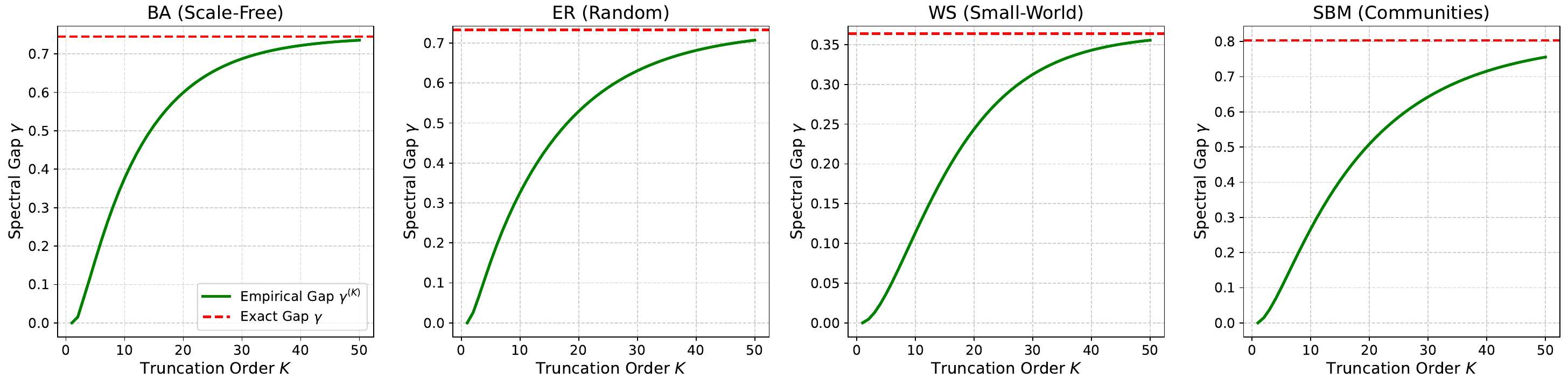} 
    \caption{Convergence of the empirical spectral gap $\gamma^{(K)}$ of the compensated truncated DSM $\hat{B}_K$ toward the exact spectral gap $\gamma$ across four distinct graph topologies ($N=1000$). At low truncation orders ($K$), the empirical gap remains significantly smaller than the exact gap, indicating a deliberately restricted mixing time. As $K \to \infty$, the gap monotonically converges to the theoretical topological bound derived in Theorem \ref{thm:spectral_gap}.}
    \label{fig:spectral_gap}
\end{figure*}

To empirically validate Theorem \ref{thm:spectral_gap}, Figure \ref{fig:spectral_gap} illustrates the convergence of the empirical spectral gap $\gamma^{(K)} = 1 - \mu_2^{(K)}$ for the compensated truncated matrix $\hat{B}_K$. At low truncation orders, $\gamma^{(K)}$ is heavily attenuated ($\mu_2^{(K)} \approx 1$), demonstrating that finite truncation explicitly acts as a spectral regularizer. By restricting the global mixing rate, this attenuation slows the exponential decay of principal transient modes, inherently preserving high-frequency spatial signals and delaying over-smoothing. As $K \to \infty$, $\gamma^{(K)}$ monotonically converges to the exact topological limit dictated by the algebraic connectivity $\lambda_2$. This spectral trajectory encapsulates the fundamental trade-off between over-smoothing and over-squashing in deep Graph Neural Networks, which is rigorously bounded by the Cheeger constant. A large exact gap $\gamma$ implies a high Cheeger constant (an expansive topology lacking severe bottlenecks); while this accelerates mixing to alleviate over-squashing, it inevitably drives node representations toward an indistinguishable stationary state. Conversely, a small gap protects local structural diversity but mathematically necessitates topological bottlenecks, thereby exacerbating over-squashing. Consequently, parameterizing the diffusion operator via $K$ provides a theoretically grounded mechanism to explicitly control this spectral gap, effectively balancing localized feature retention against global steady-state mixing.

% ----------------------------------------------------------------------
% Section 5: Truncated Neumann Series
% ----------------------------------------------------------------------
\section{Truncated Neumann Series}
\label{sec:truncated_neumann}

While the exact doubly stochastic graph matrix (DSM) $B = \tilde{L}^{-1}$ captures the global topology of the graph, computing the exact inverse requires $O(n^3)$ operations. To make this tractable for large-scale graph neural networks, we truncate the infinite Neumann series to a finite order $K$. Let the transition matrix be $P = \tilde{D}^{-1}A$. The $K$-th order truncated DSM is defined as:
\begin{equation}
B_K = \sum_{k=0}^{K} P^k \tilde{D}^{-1}
\label{eq:trunc_b_k}
\end{equation}

Beyond reducing the computational complexity to $O(K|E|)$, this truncation serves a fundamental regularizing purpose in deep representation learning. The exact DSM assumes a receptive field of the entire graph, which inevitably causes node features to exponentially mix toward indistinguishable states, exacerbating the over-smoothing problem. By truncating the series at step $K$, we mathematically force the attention weights between nodes with a shortest-path distance greater than $K$ to be strictly zero. This imposes a hard localized receptive field, allowing the network to capture essential structural properties without over-mixing. The truncation process modifies the original properties of the exact DSM in the following predictable ways:

\begin{itemize}
    \item \textbf{Localized Proximity and Bounded Receptive Field:} The exact DSM exhibits strict diagonal dominance and exponential distance decay \cite{merris1997doubly, andrade2025doubly}. Truncating to $K$ terms preserves these properties within a strict $K$-hop boundary. Isolating the individual entries of $B_K$ yields:
    \begin{equation}
    b_{ij}^{(K)} = 
    \begin{cases} 
    \frac{1}{d_i+1} \left( 1 + \sum_{k=1}^K (P^k)_{ii} \right) & \text{if } i = j \\ 
    \frac{1}{d_j+1} \sum_{k=1}^K (P^k)_{ij} & \text{if } i \neq j 
    \end{cases}
    \label{eq:trunc_b_ij}
    \end{equation}
    Since $\|P\|_\infty < 1$, the transition probabilities strictly decay with $k$. The zero-th order term $\frac{1}{d_i+1}$ provides a massive static mass exclusively to the self-loop, guaranteeing $b_{ii}^{(K)}$ remains the strictly dominant element in its row. Crucially, by the algebraic properties of transition matrices, $(P^k)_{ij} = 0$ if no path of length $k$ exists between $i$ and $j$. Consequently, $b_{ij}^{(K)}$ is mathematically forced to zero for any topological distance exceeding $K$, structurally preventing the integration of global noise.

    \item \textbf{Local Centrality Indicator:} While the exact $b_{ii}$ measures global centrality \cite{andrade2025doubly}, Equation \ref{eq:trunc_b_ij} demonstrates that $b_{ii}^{(K)}$ functions as a rigorous $K$-hop local centrality metric. A small $b_{ii}^{(K)}$ strictly requires two structural conditions: a large local degree $d_i$ and a low return probability sum $\sum_{k=1}^K (P^k)_{ii}$. Topologically, this identifies nodes situated in expansive local neighborhoods where random walks rapidly diffuse outward rather than being trapped in local cycles.
\end{itemize}

To guarantee that the truncated matrix $B_K$ is a reliable proxy for the exact DSM, we must establish a rigorous upper bound on the approximation error.

\begin{theorem}[Truncation Error Bound]
\label{thm:trunc_error}
Let $B$ be the exact doubly stochastic graph matrix and $B_K$ be its $K$-th order Neumann series approximation. Let $d_{max}$ be the maximum node degree in the graph. The infinity norm of the approximation error $E_K = B - B_K$ decays exponentially with respect to $K$, bounded by:
\begin{equation}
\|E_K\|_\infty \le \left( \frac{d_{max}}{d_{max}+1} \right)^{K+1}
\label{eq:trunc_bound}
\end{equation}
\end{theorem}
\textit{Proof.} See Appendix \ref{app:proof_theorem3}.

\subsection{Residual Mass Compensation: Restoring Topological Conservation}
\label{subsec:residual_mass}

While Theorem \ref{thm:trunc_error} guarantees that the global truncation error strictly bounds the difference between $B$ and $B_K$, the physical distribution of this error across the graph topology is highly non-uniform. By truncating the infinite series, the diffusion system becomes algebraically \textbf{open}, causing probability mass to leak. To precisely quantify this leakage, we analyze the row sums of the truncated matrix $B_K$ by multiplying it with the all-ones vector $\mathbf{1}$. Recall that the transition matrix is $P = \tilde{D}^{-1}A$ and the node degree vector is $A\mathbf{1} = \mathbf{d}$. Since $\tilde{D} = I + D$, we have $\mathbf{1} = \tilde{D}^{-1}(I + D)\mathbf{1} = \tilde{D}^{-1}\mathbf{1} + \tilde{D}^{-1}\mathbf{d}$. This algebraic relation implies $\tilde{D}^{-1}\mathbf{1} = \mathbf{1} - \tilde{D}^{-1}\mathbf{d} = (I - P)\mathbf{1}$. Substituting this identity into the definition of $B_K$ yields a perfect telescoping sum:
\begin{equation}
B_K \mathbf{1} = \sum_{k=0}^K P^k \tilde{D}^{-1} \mathbf{1} = \sum_{k=0}^K P^k (I - P) \mathbf{1} = \mathbf{1} - P^{K+1} \mathbf{1}
\label{eq:mass_leak}
\end{equation}

The term $P^{K+1} \mathbf{1}$ represents the exact residual mass, which is the probability that a random walk of length $K+1$ has not yet terminated. In the exact doubly stochastic graph matrix, the termination probability at any step is strictly $\frac{1}{d_i+1}$. Consequently, highly central nodes with large degrees possess extremely low termination probabilities. When the series is forcibly truncated at step $K$, these central nodes suffer the largest mass loss, artificially diluting their self-weight and heavily penalizing their topological centrality.

To correct this structural bias without abandoning the hard localized receptive field, we introduce a rigorous Residual Mass Compensation mechanism. Rather than discarding the infinite tail $E_K = \sum_{k=K+1}^{\infty} P^k \tilde{D}^{-1}$, we analytically fold its total residual mass back into the node's origin (the self-loop). The compensated truncated DSM, denoted as $\hat{B}_K$, is defined as:
\begin{equation}
\hat{B}_K = B_K + \text{diag}(P^{K+1} \mathbf{1})
\label{eq:b_hat_k}
\end{equation}

This mathematically elegant compensation provides three critical theoretical guarantees: 1) \textbf{Strict row-stochasticity}, as the effective operator regains absolute mass conservation ($\hat{B}_K \mathbf{1} = (I - P^{K+1})\mathbf{1} + P^{K+1}\mathbf{1} = \mathbf{1}$), perfectly aligning with the normalization requirements of deep graph representation learning; 2) \textbf{Centrality restoration}, because the mass re-injection ensures the diagonal entry $\hat{b}_{ii}$ remains strictly dominant, explicitly preserving the centrality prior of the nodes; and 3) \textbf{Computational efficiency}. Computing the dense matrix power $P^{K+1}$ explicitly would require an intractable $O(n^3)$ operations. However, by exploiting the associativity of matrix multiplication, the residual mass vector is computed from right to left via sequential sparse matrix-vector multiplications (SpMV). Specifically, by initializing $m^{(0)} = \mathbf{1}$, we iteratively compute $m^{(k)} = P m^{(k-1)}$ in parallel with the feature propagation. This strictly bounds the time complexity to $O(K|E|)$, entirely avoiding dense matrix operations.

Empirical results rigorously corroborate this mathematical framework (detailed verification is provided in Appendix \ref{app:empirical_verification}). The actual truncation error is strictly bounded below the theoretical limit derived in Theorem \ref{thm:trunc_error}, showing tight convergence on ER graphs and dropping significantly faster than the worst-case bound on scale-free BA graphs. While the compensated error is slightly higher than the uncompensated error, this numerical discrepancy is a necessary algebraic trade-off to restore absolute probability mass conservation and diagonal structural dominance. Furthermore, the $O(K|E|)$ SpMM-based propagation remains highly scalable, executing in a fraction of a second even at $K=100$, completely circumventing the severe polynomial explosion of the $O(n^3)$ exact matrix inversion as the graph size scales.

% ----------------------------------------------------------------------
% Section 6: Model Architecture
% ----------------------------------------------------------------------
\section{Model Architecture}
\label{sec:architecture}

To effectively utilize the enhanced doubly stochastic graph matrix in deep representation learning, we establish a rigorous architectural framework based on the decoupled graph neural network paradigm. In this section, we first introduce the base decoupled formulation. We then present \textbf{DsmNet}, our foundational model that leverages a truncated Neumann series to execute computationally efficient, strictly localized multi-hop diffusion. Finally, to resolve the algebraic probability mass leakage inherently caused by truncation, we propose \textbf{DsmNet-compensate}, an advanced variant that iteratively tracks and re-injects this leaked mass during the forward pass, thereby perfectly restoring node centrality and absolute row-stochasticity.

\subsection{Decoupled Graph Neural Networks}
To address the limitations of coupled layer-wise architectures, decoupled models, initially popularized by methods such as APPNP \cite{gasteiger_predict_2019}, separate the neural feature transformation from the topological propagation. The architecture operates in two distinct phases. First, a local neural network generates base predictions solely from node features:
\begin{equation}
Z^{(0)} = f_\theta(X)
\label{eq:decoupled_mlp}
\end{equation}
Second, these base predictions are propagated through the graph using a linear diffusion operator $\Pi$. The final output $Z$ is computed as $Z = \Pi Z^{(0)}$. This design ensures the model captures $K$-hop topological dependencies without exponentially increasing trainable parameters or causing severe gradient vanishing.

\subsection{DsmNet: Truncated DSM Diffusion}
Our foundational model, DsmNet, utilizes the $K$-th order truncated Neumann series $B_K$ as the exact diffusion operator. Let the transition matrix be $P = \tilde{D}^{-1}A$. The forward propagation is defined as:
\begin{equation}
Z = B_K Z^{(0)} = \left( \sum_{k=0}^K P^k \tilde{D}^{-1} \right) Z^{(0)}
\label{eq:dsmnet_forward}
\end{equation}
This formulation guarantees an $O(K|E|)$ computational complexity by executing sequential sparse matrix multiplications, enforcing a hard localized receptive field strictly bounded by $K$ hops.

\subsection{DsmNet-compensate: Residual Mass Compensation}
Because the finite truncation in DsmNet mathematically leads to probability mass leakage, which arbitrarily penalizes the structural weights of highly central nodes, we introduce DsmNet-compensate. This variant integrates the Residual Mass Compensation mechanism, deploying $\hat{B}_K$ as the propagation operator:
\begin{equation}
Z = \hat{B}_K Z^{(0)} = \left( \sum_{k=0}^K P^k \tilde{D}^{-1} + \text{diag}(P^{K+1} \mathbf{1}) \right) Z^{(0)}
\label{eq:dsmnet_comp_forward}
\end{equation}
To maintain strict linear computational complexity $O(K|E|)$ and completely avoid dense matrix operations, we compute this forward pass through synchronous iterative state updates using exclusively sparse matrix operations. First, we initialize the feature state $S^{(0)} = \tilde{D}^{-1} Z^{(0)}$ and the mass tracking vector $m^{(0)} = \mathbf{1}$. For step $k = 1, \dots, K$, we update both states synchronously:
\begin{align}
S^{(k)} &= P S^{(k-1)} + S^{(0)} \label{eq:update_s} \\
m^{(k)} &= P m^{(k-1)} \label{eq:update_m}
\end{align}
After $K$ iterations, the final representation is obtained by folding the exact residual mass tail $m^{(K+1)} = P m^{(K)}$ back into the initial predictions:
\begin{equation}
Z = S^{(K)} + \text{diag}(m^{(K+1)}) Z^{(0)}
\label{eq:final_z}
\end{equation}
This mathematically closed formulation catches escaping random walk probabilities at the $K$-hop boundary and re-injects them into the ego-node, restoring absolute probability mass conservation.

\subsection{Intuitive Explanation of the Forward Pass}
Intuitively, this synchronous iterative process represents a strict steady-state diffusion governed by a mathematically closed system. The feature propagation sequence, $S^{(k)} = P S^{(k-1)} + S^{(0)}$, continuously re-injects the node's original localized features while pushing accumulated information outward. This mechanism forces features to traverse exactly $K$ hops to capture neighborhood structures, while the continuous re-injection prevents the signal from being entirely diluted by distant noise. Simultaneously, the mass tracking vector $m^{(k)}$ acts as a rigorous topological accounting system. Because early truncation forces random walks to terminate prematurely, the probability mass naturally leaks out of the algebraic system. The operation $\text{diag}(m^{(K+1)}) Z^{(0)}$ functions as a topological sink: it accurately aggregates all the escaping random walk probabilities at the $K$-hop boundary and immediately folds them back into the ego-node's self-loop. This ensures that highly central nodes, which intrinsically possess large residual mass due to their expansive connectivity, correctly retain their structural dominance. The resulting framework perfectly mimics the local proximity and centrality properties of the exact doubly stochastic graph matrix at a fraction of the computational cost.

% ----------------------------------------------------------------------
% Section 7: Experiments
% ----------------------------------------------------------------------
\section{Empirical Evaluation}
\label{sec:experiments}

To rigorously validate the theoretical advantages of the truncated doubly stochastic graph matrix, we design a comprehensive set of experiments. Our evaluation aims to verify whether the decoupled architecture can prevent over-smoothing, capture long-range dependencies, and scale efficiently to large graphs.

\subsection{Datasets}
To rigorously evaluate the efficacy of the proposed DsmNet, we utilize a comprehensive suite of benchmark datasets characterized by diverse topological properties. For homophilic graphs \cite{zhu2020beyond, zheng2024missing}, where connected nodes tend to share similar labels, we include the Planetoid citation networks \cite{yang2016revisiting} (Cora, CiteSeer, and PubMed), the Amazon co-purchase networks \cite{shchur2018pitfalls} (Computers and Photo), and the Coauthor networks \cite{shchur2018pitfalls} (CS and Physics). To assess the model's robustness under heterophily \cite{luan2024heterophily, luan2024heterophilic}, we employ the WebKB web-page datasets \cite{pei2020geom} (Texas, Cornell, and Wisconsin) and the Wikipedia networks \cite{rozemberczki2021multi} (Chameleon and Squirrel), where edge connections often bridge nodes from different classes. 

\subsection{Baseline Models}
To provide a comprehensive evaluation of the proposed DsmNet, we compare its performance against several representative baseline models that have defined the evolution of Graph Neural Networks. ChebNet \cite{defferrard2016convolutional} implements fast localized spectral filters through the use of Chebyshev polynomials to circumvent the high computational costs of eigendecompositions. GCN \cite{kipf2017semi} further simplifies this framework using a first-order approximation of spectral graph convolutions, resulting in an efficient, scalable model for semi-supervised classification. For inductive tasks, GraphSAGE \cite{hamilton2017inductive} introduces a framework that generates embeddings by sampling and aggregating features from local neighborhoods rather than training individual node parameters. GAT \cite{velickovic2018graph} enhances these spatial approaches by incorporating self-attentional layers, which allow the model to implicitly assign varying importance to different neighbors. To better adapt to local graph structures, JKNet \cite{xu2018representation} utilizes jumping knowledge connections to flexibly combine representations from different neighborhood ranges. APPNP \cite{gasteiger_predict_2019} separates the neural feature transformation from the topological propagation by employing a Personalized PageRank scheme, which facilitates a larger, adjustable receptive field without increasing the number of trainable parameters. Finally, GCNII \cite{chenWHDL2020gcnii} extends the vanilla GCN to deep architectures by utilizing initial residual connections and identity mapping to effectively mitigate the over-smoothing problem encountered in multi-layer networks.

\subsection{Experiment Settings}
All experiments are conducted on an NVIDIA L40S GPU. We report the mean and standard deviation of the accuracy metrics across 5 independent random seeds. The neural transformation module consists of two linear layers with a hidden size of 64. The truncation order $K$ is tuned within the set $\{10, 20, 30, 40, 50\}$, and the learning rate is optimized within the range $[10^{-5}, 10^{-3}]$ using the Optuna framework \cite{ozaki2025optunahub}. We evaluate the models under two rigorous data split protocols: semi-supervised and fully supervised. In the semi-supervised setting, we sample 20 nodes per class for the training set. For the smaller Texas, Wisconsin, and Cornell datasets, the training set is restricted to 10 nodes per class. In all semi-supervised scenarios, we allocate 20\% of the remaining nodes for validation and 20\% for testing. In the fully supervised setting, the datasets are randomly partitioned into 60\% for training, 20\% for validation, and 20\% for testing.

% --- Homophilic Tables ---
\begin{table*}[htbp]
\centering
\caption{Node classification accuracy (percentage) on homophilic datasets under \textbf{Full Supervised (FULL)} setting. Standard deviations are also scaled. Legend: \colorbox{first}{\scriptsize 1st}, \colorbox{second}{\scriptsize 2nd}, \colorbox{third}{\scriptsize 3rd}.}
\label{tab:homophilic_full}
\resizebox{\textwidth}{!}{
\begin{tabular}{lccccccc}
\toprule
\textbf{Model} & \textbf{Cora} & \textbf{CiteSeer} & \textbf{PubMed} & \textbf{CS} & \textbf{Physics} & \textbf{Photo} & \textbf{Computers} \\
\midrule
GCN       & 87.8 ± 1.2 & 75.5 ± 1.3 & 86.8 ± 0.6 & 89.3 ± 0.4 & 95.4 ± 0.2 & 84.0 ± 1.1 & 66.1 ± 0.9 \\
GAT       & 87.4 ± 0.7 & 76.1 ± 0.8 & 85.6 ± 0.6 & 89.4 ± 0.4 & 95.3 ± 0.2 & 83.7 ± 1.5 & 71.4 ± 1.0 \\
SAGE      & 87.6 ± 0.6 & \cellcolor{third}76.6 ± 1.4 & 89.1 ± 0.4 & \cellcolor{third}93.8 ± 0.6 & \cellcolor{third}96.2 ± 0.2 & 90.9 ± 0.8 & 78.2 ± 0.8 \\
CHEB      & \cellcolor{third}88.2 ± 0.6 & \cellcolor{second}76.8 ± 1.1 & \cellcolor{third}89.6 ± 0.4 & 93.7 ± 0.5 & OOM & 90.6 ± 0.8 & \cellcolor{second}79.7 ± 1.3 \\
APPNP     & 88.1 ± 1.1 & 76.4 ± 1.2 & 75.8 ± 0.6 & 88.0 ± 0.5 & 95.1 ± 0.2 & 82.1 ± 2.6 & 77.9 ± 2.5 \\
GCNII     & 86.4 ± 1.1 & 73.9 ± 1.1 & 85.3 ± 0.4 & 82.2 ± 0.7 & 94.6 ± 0.3 & \cellcolor{second}91.1 ± 0.5 & 78.1 ± 2.9 \\
JKNET     & 85.9 ± 0.9 & 72.7 ± 0.9 & 87.3 ± 0.2 & 92.3 ± 0.6 & 96.1 ± 0.0 & \cellcolor{third}91.0 ± 1.6 & \cellcolor{first}82.4 ± 0.8 \\
\midrule
DSM       & \cellcolor{first}88.8 ± 1.3 & \cellcolor{third}76.6 ± 0.9 & \cellcolor{first}90.3 ± 0.4 & \cellcolor{second}94.5 ± 0.4 & \cellcolor{second}96.9 ± 0.1 & 90.4 ± 1.4 & \cellcolor{third}78.3 ± 2.0 \\
DSM COMP. & \cellcolor{second}88.5 ± 1.1 & \cellcolor{first}77.1 ± 1.2 & \cellcolor{first}90.3 ± 0.4 & \cellcolor{first}94.9 ± 0.3 & \cellcolor{first}97.0 ± 0.2 & \cellcolor{first}92.0 ± 0.7 & 77.7 ± 2.7 \\
\bottomrule
\end{tabular}
}
\end{table*}

\begin{table*}[htbp]
\centering
\caption{Node classification accuracy (percentage) on homophilic datasets under \textbf{Semi-Supervised (SEMI)} setting. Standard deviations are also scaled. Legend: \colorbox{first}{\scriptsize 1st}, \colorbox{second}{\scriptsize 2nd}, \colorbox{third}{\scriptsize 3rd}.}
\label{tab:homophilic_semi}
\resizebox{\textwidth}{!}{
\begin{tabular}{lccccccc}
\toprule
\textbf{Model} & \textbf{Cora} & \textbf{CiteSeer} & \textbf{PubMed} & \textbf{CS} & \textbf{Physics} & \textbf{Photo} & \textbf{Computers} \\
\midrule
GCN       & 78.6 ± 1.7 & 68.6 ± 1.4 & 78.4 ± 2.0 & 91.0 ± 0.6 & 93.5 ± 0.8 & 86.1 ± 2.4 & 70.5 ± 2.8 \\
GAT       & 80.0 ± 1.6 & 68.7 ± 2.3 & 78.2 ± 2.0 & 90.8 ± 0.6 & \cellcolor{third}93.9 ± 0.7 & 87.4 ± 1.2 & \cellcolor{third}75.0 ± 2.3 \\
SAGE      & 77.3 ± 2.4 & 69.1 ± 1.4 & 75.8 ± 2.0 & \cellcolor{second}91.8 ± 0.8 & 93.4 ± 0.8 & \cellcolor{second}90.1 ± 1.7 & \cellcolor{first}79.4 ± 1.3 \\
CHEB      & 77.9 ± 2.1 & 68.7 ± 1.1 & 69.9 ± 1.5 & 89.8 ± 1.1 & OOM & 81.7 ± 3.1 & 61.5 ± 5.8 \\
APPNP     & \cellcolor{third}81.2 ± 1.2 & \cellcolor{third}69.3 ± 1.9 & \cellcolor{third}79.2 ± 2.5 & \cellcolor{third}91.5 ± 0.4 & 90.1 ± 2.6 & 87.0 ± 1.3 & 67.7 ± 3.6 \\
GCNII     & 80.5 ± 1.8 & 60.8 ± 1.8 & 73.1 ± 5.3 & 80.4 ± 2.8 & 92.5 ± 0.7 & 88.1 ± 4.4 & \cellcolor{second}76.3 ± 4.1 \\
JKNET     & 77.1 ± 1.7 & 66.6 ± 2.5 & 77.9 ± 2.8 & 86.2 ± 1.3 & 91.9 ± 2.5 & 84.3 ± 3.7 & 66.9 ± 4.6 \\
\midrule
DSM       & \cellcolor{first}82.0 ± 1.4 & \cellcolor{second}69.6 ± 1.8 & \cellcolor{second}79.5 ± 2.6 & 91.4 ± 1.7 & \cellcolor{first}94.4 ± 0.4 & \cellcolor{third}88.9 ± 0.8 & 48.1 ± 9.6 \\
DSM COMP. & \cellcolor{second}81.5 ± 1.4 & \cellcolor{first}70.5 ± 1.7 & \cellcolor{first}79.8 ± 2.1 & \cellcolor{first}92.5 ± 0.7 & \cellcolor{second}94.1 ± 0.3 & \cellcolor{first}90.7 ± 1.6 & 74.9 ± 2.3 \\
\bottomrule
\end{tabular}
}
\end{table*}

\begin{figure*}[htbp]
    \centering
    \includegraphics[width=\textwidth]{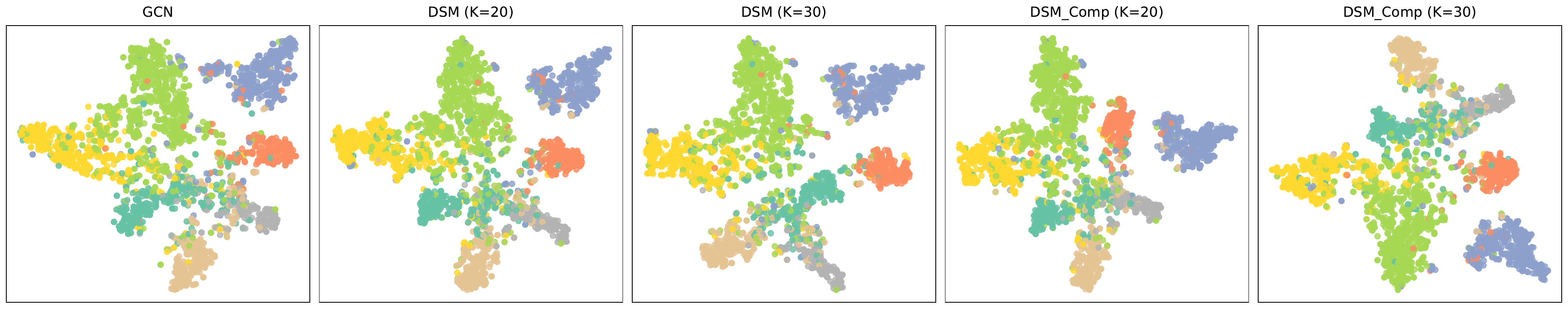}
    \caption{t-SNE visualization of the learned node representations on the homophilic Cora dataset. The latent spaces generated by the baseline GCN are compared against the uncompensated truncated DSM ($K=20, 30$) and the compensated $\hat{B}_K$ ($K=20, 30$). The visual evidence corroborates that the residual mass compensation mechanism enhances intra-class compactness while maintaining clear inter-class decision boundaries, aligning with the quantitative performance advantages.}
    \label{fig:tsne_cora}
\end{figure*}

\subsection{Performance Analysis on Homophilic and Heterophilic Graphs}
As demonstrated in Tables \ref{tab:homophilic_full} and \ref{tab:homophilic_semi}, the proposed DsmNet and DsmNet-compensate models consistently achieve highly competitive performance on homophilic datasets. In these topologies, where connected nodes are highly likely to share identical class labels, the continuous topological proximity captured by the DSM effectively aggregates relevant neighborhood signals. The mass compensation mechanism further anchors high-frequency structural features, yielding leading results across multiple benchmarks.

Conversely, the models exhibit comparatively inferior performance on heterophilic datasets, as detailed in Appendix \ref{app:heterophilic} (Tables \ref{tab:heterophilic_full} and \ref{tab:heterophilic_semi}). This empirical divergence is mathematically intrinsic to the spectral properties of the DSM. Operating fundamentally as a low-pass filter in the spectral domain, the DSM diffuses and smooths representations across connected nodes. In heterophilic graphs, where adjacent vertices frequently belong to disparate classes and exhibit high feature variance, this low-pass diffusion inadvertently compresses effective discriminative information. Consequently, the representations of topologically proximate but semantically distinct nodes become homogenized, negatively impacting classification accuracy.

\subsection{Node Centrality Prediction and Cross-Scale Generalization}
\label{subsec:centrality_task}
To evaluate the structural descriptive power and generalization capability of the doubly stochastic graph matrix (DSM), we design a regression task to predict node-level normalized betweenness centrality. A critical challenge in this experiment is \textbf{cross-scale generalization}: the model is trained on a dataset of 1,000 graphs with only $N=200$ nodes, while it is validated and tested on 300 graphs with $N=1,000$ nodes. This setup tests whether the learned structural priors can generalize to significantly larger graphs. The architecture consists of a 2-layer Graph Convolutional Network (GCN) with a hidden dimension of 64. The input features are either node degrees ($d_i$), exact or truncated DSM diagonal entries ($b_{ii}$), or their concatenated combinations. We utilize a \textit{Pairwise Ranking Loss} to optimize the model's ability to preserve the relative centrality ranking of nodes. Table \ref{tab:feature_ablation} presents the experimental results. All Hit Ratio (HR) and Normalized Discounted Cumulative Gain (NDCG) metrics are expressed as percentages with the standard deviation ($\pm$) across 5 random seeds, rounded to one decimal place. The results clearly indicate that incorporating DSM diagonal features provides a substantial performance boost, especially in Watts–Strogatz (WS) and Random Geometric (RGG) graphs where local degree information is a poor proxy for global betweenness. Notably, the DSM-based features maintain high ranking correlation ($\tau$) even when the graph size increases fivefold from training to testing, proving that $b_{ii}$ captures scale-invariant topological properties.

\begin{table*}[htbp]
\centering
\footnotesize
\caption{Betweenness Centrality ranking performance under cross-scale generalization ($N_{train}=200, N_{test}=1000$). $\tau$ denotes Kendall's Tau, while HR and NDCG are expressed as percentages. Standard deviations are calculated over 5 seeds. Legend: \colorbox{first}{\scriptsize 1st}, \colorbox{second}{\scriptsize 2nd}, \colorbox{third}{\scriptsize 3rd}.}
\label{tab:feature_ablation}
\resizebox{\textwidth}{!}{
\begin{tabular}{llccccc}
\toprule
\textbf{Dataset} & \textbf{Feature Mode} & \textbf{Kendall's $\tau$} & \textbf{HR@30 (\%)} & \textbf{HR@50 (\%)} & \textbf{NDCG@30 (\%)} & \textbf{NDCG@50 (\%)} \\
\midrule
\multirow{7}{*}{\textbf{WS}} 
& Degree Only              & 0.553 ± 0.005 & 49.5 ± 0.4 & 54.6 ± 0.2 & 84.5 ± 0.2 & 85.2 ± 0.1 \\
& Exact DSM               & 0.756 ± 0.003 & 62.9 ± 0.9 & 67.3 ± 0.6 & 89.6 ± 1.0 & 90.9 ± 0.7 \\
& Truncated DSM           & 0.756 ± 0.004 & 62.6 ± 0.5 & 67.1 ± 0.3 & 89.4 ± 0.9 & 90.8 ± 0.6 \\
& Comp. Trunc. DSM        & 0.755 ± 0.003 & 62.7 ± 0.9 & 67.2 ± 0.4 & 89.5 ± 1.0 & 90.9 ± 0.7 \\
& Deg + Exact DSM         & \cellcolor{second}0.769 ± 0.002 & \cellcolor{second}66.5 ± 0.4 & \cellcolor{first}69.0 ± 0.2 & \cellcolor{second}91.6 ± 0.4 & \cellcolor{first}92.3 ± 0.2 \\
& Deg + Truncated DSM     & \cellcolor{third}0.769 ± 0.002 & \cellcolor{third}66.3 ± 0.4 & \cellcolor{second}68.9 ± 0.4 & \cellcolor{third}91.5 ± 0.3 & \cellcolor{third}92.2 ± 0.2 \\
& Deg + Comp. Trunc. DSM  & \cellcolor{first}0.770 ± 0.002 & \cellcolor{first}66.5 ± 0.5 & \cellcolor{third}68.9 ± 0.2 & \cellcolor{first}91.6 ± 0.4 & \cellcolor{second}92.3 ± 0.2 \\
\midrule
\multirow{7}{*}{\textbf{RGG}} 
& Degree Only              & 0.465 ± 0.004 & 26.1 ± 0.4 & 37.8 ± 0.3 & 37.5 ± 0.6 & 42.9 ± 0.5 \\
& Exact DSM               & 0.489 ± 0.002 & 28.0 ± 0.2 & 39.6 ± 0.3 & 39.2 ± 0.3 & 44.9 ± 0.4 \\
& Truncated DSM           & 0.487 ± 0.004 & 27.8 ± 0.2 & 39.4 ± 0.3 & 39.1 ± 0.4 & 44.7 ± 0.4 \\
& Comp. Trunc. DSM        & 0.490 ± 0.004 & 27.8 ± 0.5 & 39.4 ± 0.4 & 39.1 ± 0.6 & 44.7 ± 0.6 \\
& Deg + Exact DSM         & \cellcolor{third}0.602 ± 0.002 & \cellcolor{third}34.0 ± 0.2 & \cellcolor{third}46.9 ± 0.3 & \cellcolor{third}46.5 ± 0.3 & \cellcolor{third}52.5 ± 0.3 \\
& Deg + Truncated DSM     & \cellcolor{second}0.602 ± 0.002 & \cellcolor{second}34.0 ± 0.3 & \cellcolor{second}47.0 ± 0.2 & \cellcolor{second}46.6 ± 0.3 & \cellcolor{second}52.5 ± 0.3 \\
& Deg + Comp. Trunc. DSM  & \cellcolor{first}0.602 ± 0.003 & \cellcolor{first}34.2 ± 0.3 & \cellcolor{first}47.0 ± 0.3 & \cellcolor{first}46.7 ± 0.3 & \cellcolor{first}52.5 ± 0.4 \\
\midrule
\multirow{7}{*}{\textbf{LFR}} 
& Degree Only              & 0.642 ± 0.007 & \cellcolor{first}80.1 ± 0.1 & 70.2 ± 0.2 & \cellcolor{first}98.1 ± 0.1 & \cellcolor{first}96.9 ± 0.1 \\
& Exact DSM               & 0.707 ± 0.004 & 78.0 ± 1.2 & 70.1 ± 0.4 & 95.9 ± 1.3 & 95.0 ± 1.1 \\
& Truncated DSM           & 0.707 ± 0.004 & 78.1 ± 0.9 & 70.1 ± 0.3 & 96.1 ± 0.9 & 95.2 ± 0.8 \\
& Comp. Trunc. DSM        & 0.707 ± 0.004 & 78.0 ± 1.2 & 70.1 ± 0.4 & 95.9 ± 1.3 & 95.0 ± 1.2 \\
& Deg + Exact DSM         & \cellcolor{third}0.709 ± 0.002 & \cellcolor{second}79.1 ± 0.9 & \cellcolor{second}70.4 ± 0.3 & \cellcolor{second}97.4 ± 0.4 & \cellcolor{second}96.4 ± 0.3 \\
& Deg + Truncated DSM     & \cellcolor{second}0.709 ± 0.002 & \cellcolor{third}79.0 ± 0.9 & \cellcolor{first}70.4 ± 0.3 & \cellcolor{third}97.4 ± 0.4 & \cellcolor{third}96.4 ± 0.3 \\
& Deg + Comp. Trunc. DSM  & \cellcolor{first}0.709 ± 0.002 & 79.0 ± 1.0 & \cellcolor{third}70.4 ± 0.4 & 97.4 ± 0.4 & 96.4 ± 0.3 \\
\bottomrule
\end{tabular}
}
\end{table*}

\section{Conclusion}
\label{sec:conclusion}
This paper establishes the Doubly Stochastic graph Matrix (DSM) as a mathematically rigorous alternative to the standard Laplacian operator in Graph Neural Networks. To bypass the $O(n^3)$ complexity of exact matrix inversion, we developed a scalable $O(K|E|)$ truncated Neumann series approximation. Crucially, we identified and resolved the resulting probability mass leakage via a novel Residual Mass Compensation mechanism, which strictly preserves row-stochasticity and node centrality. Both theoretical and empirical analyses demonstrate that our compensated architecture explicitly bounds Dirichlet energy decay, effectively mitigating catastrophic over-smoothing while enforcing a hard localized receptive field. The proposed models achieve highly competitive accuracy on homophilic datasets and serve as robust, continuous structural encodings for generalized Graph Transformers. Future work will focus on the dynamic optimization of the truncation order $K$ for varying local topologies, as well as extending the DSM framework to directed and dynamic graphs.

\bibliographystyle{plainnat}
\bibliography{reference}  

% ----------------------------------------------------------------------
% Appendix
% ----------------------------------------------------------------------
\appendix
\section*{Supplementary Material}

\section{Proofs of Theoretical Results}
\label{app:proofs}

\subsection{Proof of Lemma \ref{lem:eigen_b}}
\label{app:proof_lemma1}
Let $v$ be an eigenvector of $L$ corresponding to the eigenvalue $\lambda$. By definition, we have $L v = \lambda v$. Adding the identity matrix to the operator yields $(I + L) v = (1 + \lambda) v$. Since $L$ is positive semidefinite, $\lambda \ge 0$, which guarantees that $(1 + \lambda) \ge 1$, meaning $(I + L)$ is strictly positive definite and invertible. Multiplying both sides by $(I + L)^{-1}$ provides $(I + L)^{-1} v = \frac{1}{1+\lambda} v$. Thus, every eigenvector of $L$ is an eigenvector of $B$ with the eigenvalue mapped through the strictly decreasing function $f(x) = \frac{1}{1+x}$. Consequently, the ascending order of the Laplacian eigenvalues translates to a descending order of the doubly stochastic matrix eigenvalues: $1 = \mu_1 \ge \mu_2 \ge \dots \ge \mu_n > 0$. \qed

\subsection{Proof of Theorem \ref{thm:spectral_gap}}
\label{app:proof_theorem2}
The matrix $B$ operates as a discrete-time transition matrix of a random walk, acting as an averaging operator rather than a differential one. When $B$ is repeatedly applied to a signal $x$, the operation $B^k x$ projects the signal onto the eigenspace of $B$. The component of $x$ aligned with the principal eigenvector $v_1$ (associated with $\mu_1 = 1$) remains completely unchanged, representing the stationary distribution. The orthogonal transient components decay at rates proportional to $\mu_i^k$. As $k \to \infty$, the survival of the largest transient component is strictly bounded by the second largest eigenvalue, $\mu_2$. Because $B$ is positive definite, all its eigenvalues are positive, meaning $|\mu_2| = \mu_2$. Therefore, the asymptotic rate of exponential decay toward the stationary state is governed exclusively by the separation between the dominant stationary mode and the slowest decaying transient mode, which is exactly $1 - \mu_2$. Substituting the result from Lemma \ref{lem:eigen_b} yields $\gamma = 1 - \frac{1}{1+\lambda_2}$. \qed

\subsection{Proof of Theorem \ref{thm:trunc_error}}
\label{app:proof_theorem3}
By definition, the exact matrix is the infinite sum $B = \sum_{k=0}^{\infty} P^k \tilde{D}^{-1}$. The error matrix $E_K$ is the residual tail of the series:
\begin{equation}
E_K = \sum_{k=K+1}^{\infty} P^k \tilde{D}^{-1}
\end{equation}
We can factor out $P^{K+1}$ from the infinite sum:
\begin{equation}
E_K = P^{K+1} \sum_{m=0}^{\infty} P^m \tilde{D}^{-1} = P^{K+1} B
\end{equation}
Taking the infinity norm (the maximum absolute row sum) on both sides and applying the sub-multiplicative property of matrix norms yields:
\begin{equation}
\|E_K\|_\infty \le \|P\|_\infty^{K+1} \|B\|_\infty
\end{equation}
Since $B$ is an exact doubly stochastic matrix, the sum of every row is exactly $1$, meaning its infinity norm is strictly $\|B\|_\infty = 1$. Next, we evaluate the infinity norm of the transition matrix $P = \tilde{D}^{-1}A$. The $i$-th row sum of $P$ is $\frac{d_i}{d_i+1}$. Because $f(x) = \frac{x}{x+1}$ is a strictly increasing function for $x \ge 0$, the maximum row sum occurs at the node with the maximum degree $d_{max}$. Thus:
\begin{equation}
\|P\|_\infty = \frac{d_{max}}{d_{max}+1}
\end{equation}
Substituting these values back into the inequality, we obtain the absolute upper bound:
\begin{equation}
\|E_K\|_\infty \le \left( \frac{d_{max}}{d_{max}+1} \right)^{K+1}
\end{equation}
Because $\frac{d_{max}}{d_{max}+1} < 1$, the error matrix $E_K$ approaches the zero matrix exponentially as $K \to \infty$. \qed

\section{Extended Topological Analysis}

\subsection{Empirical Verification of Approximation Error and Scalability}
\label{app:empirical_verification}

To further validate the theoretical bounds established in Theorem \ref{thm:trunc_error} and the computational advantages of our approach, we present an empirical verification of the Neumann series approximation in Figure \ref{fig:theory_verification}. The actual truncation error is strictly bounded below the theoretical limit, showing tight convergence on Erdős–Rényi (ER) graphs and dropping significantly faster than the worst-case bound on scale-free Barabási–Albert (BA) graphs. While the compensated error slightly exceeds the uncompensated error, this numerical discrepancy represents a necessary algebraic trade-off to restore absolute probability mass conservation and diagonal structural dominance. Furthermore, the wall-clock execution time analysis confirms that the $O(K|E|)$ SpMM-based propagation remains highly scalable, entirely circumventing the severe $O(n^3)$ polynomial explosion of exact matrix inversion as the graph size scales.

\begin{figure*}[htbp]
    \centering
    \includegraphics[width=\textwidth]{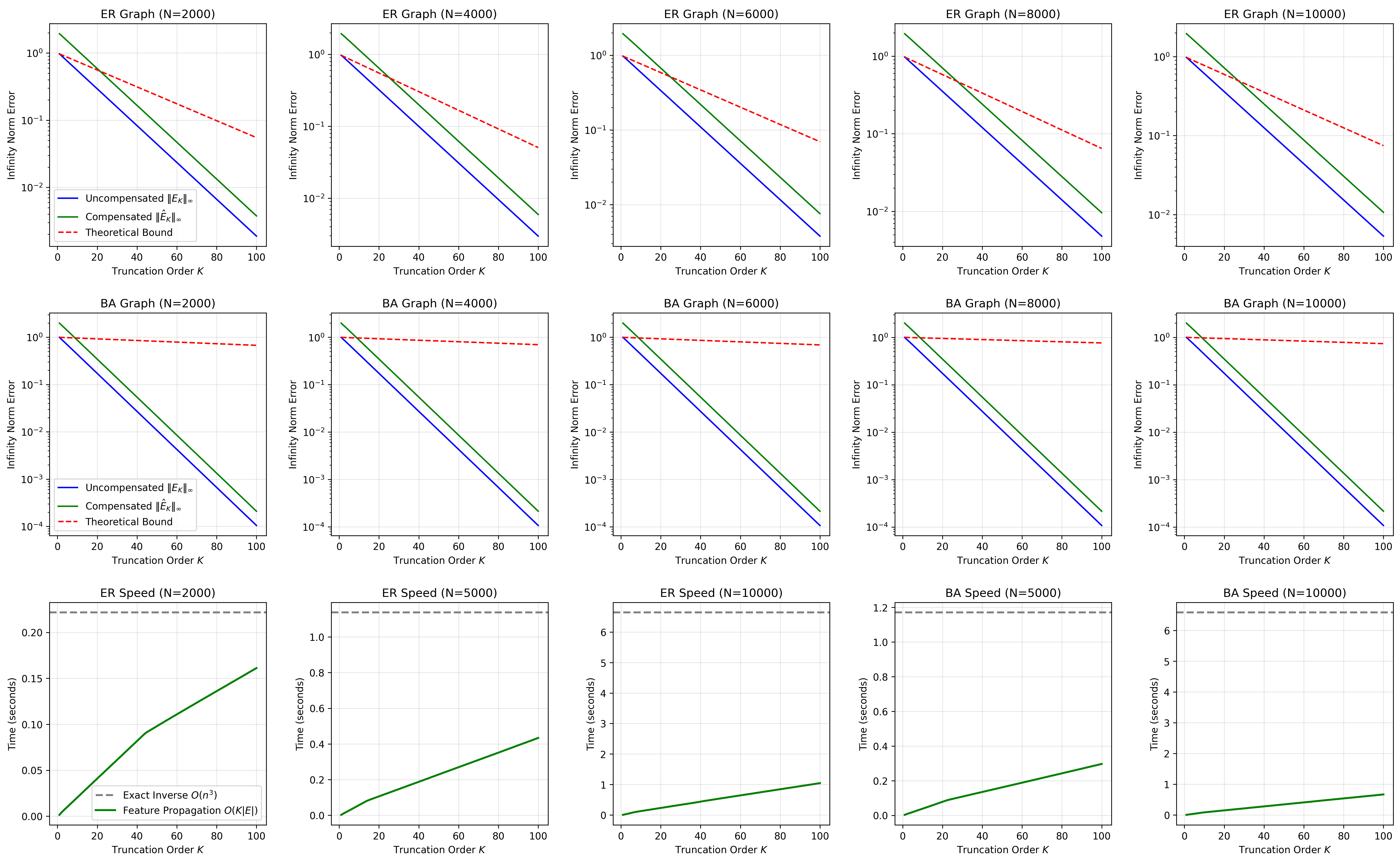} 
    \caption{Empirical verification of the Neumann series approximation. \textbf{Top \& Middle Rows:} Infinity norm truncation errors (uncompensated and compensated) compared against the theoretical upper bound (Theorem \ref{thm:trunc_error}) on ER and BA graphs. \textbf{Bottom Row:} Wall-clock execution time of the exact $O(n^3)$ dense matrix inversion versus the $O(K|E|)$ sparse feature propagation.}
    \label{fig:theory_verification}
\end{figure*}

\subsection{Rank Preservation Analysis}
\begin{figure*}[htbp]
    \centering
    \includegraphics[width=\textwidth]{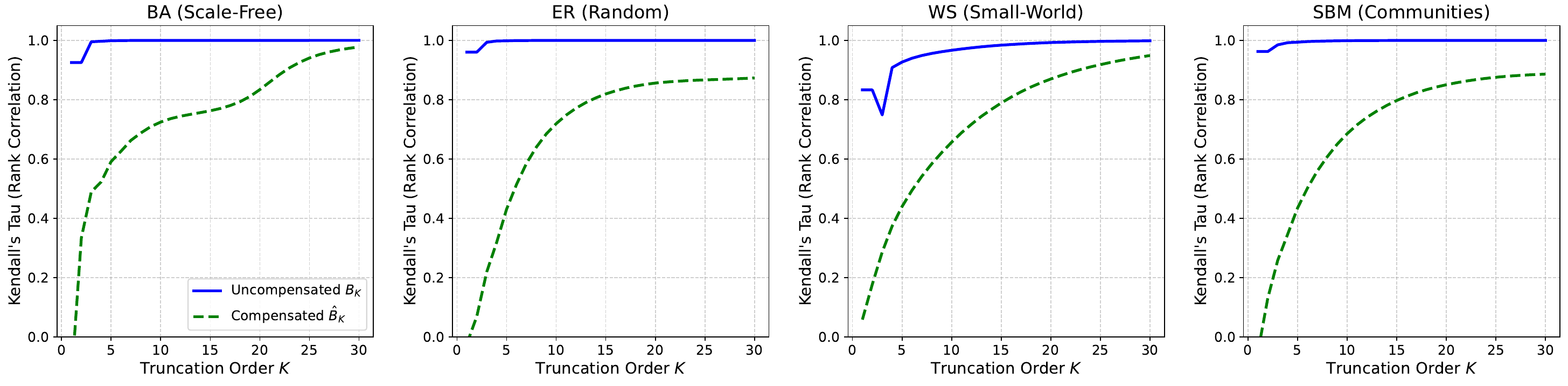}
    \caption{Node centrality rank preservation analysis using Kendall’s Tau ($\tau$) rank correlation between approximated and exact DSM diagonal entries across four synthetic topologies ($N=1000$). The blue solid line represents the uncompensated approximation ($\text{diag}(B_K)$), while the green dashed line represents the compensated approximation ($\text{diag}(\hat{B}_K)$). Higher values indicate better preservation of the relative ranking. The uncompensated $B_K$ achieves near-perfect rank preservation at extremely low truncation orders ($K<5$), whereas the compensated $\hat{B}_K$ exhibits slower convergence due to residual mass insertion, highlighting a fundamental algebraic trade-off between numerical row-stochasticity and short-range structural fidelity.}
    \label{fig:rank_preservation}
\end{figure*}
As depicted in Figure \ref{fig:rank_preservation}, we track the rank correlation of diagonal entries generated by the Neumann series approaches versus the exact inversion. The analysis empirically confirms the algebraic trade-off imposed by residual mass compensation between maintaining the strict structural ranking of vertices and conserving the total diffusion probability. 

\subsection{Distance Decay Analysis}
To empirically validate the structural implications of the truncated Neumann series, we analyze the distance decay properties of the exact and approximated doubly stochastic graph matrices. Figure \ref{fig:distance_decay} plots the mean entry value $\bar{b}_{ij}$ as a function of the shortest path distance $d(i,j)$ across four distinct graph topologies.

\begin{figure*}[htbp]
    \centering
    \includegraphics[width=\textwidth]{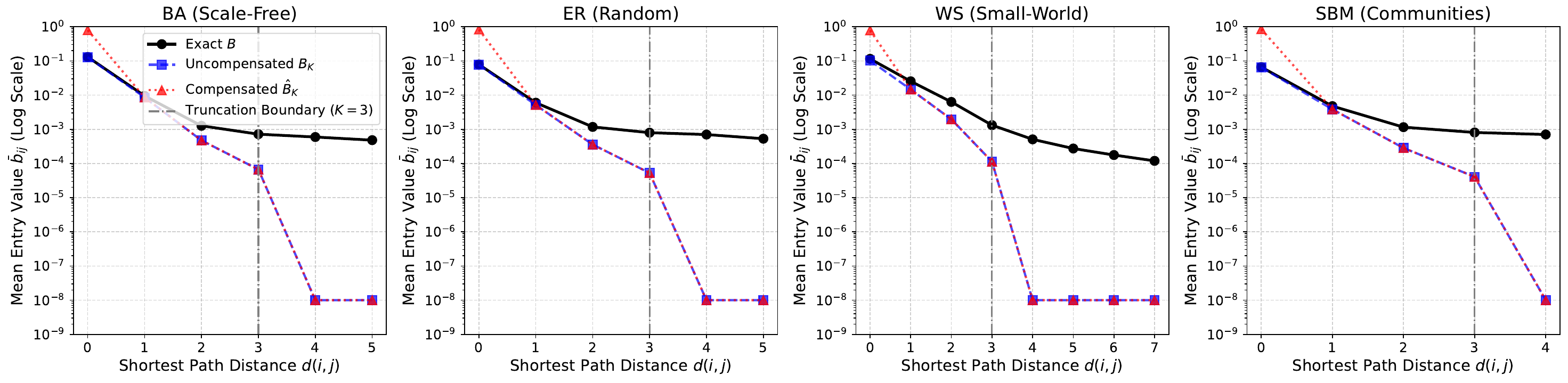}
    \caption{Distance decay of the exact doubly stochastic graph matrix $B$ versus the uncompensated ($B_K$) and compensated ($\hat{B}_K$) truncated Neumann series at order $K=3$. The approximations perfectly match the exact values within the local receptive field ($d \le 3$) and strictly enforce a hard cutoff for distances $d \ge 4$. Residual mass compensation operates exclusively on the self-loop entries ($d=0$).}
    \label{fig:distance_decay}
\end{figure*}

The exact doubly stochastic matrix $B$ exhibits a global exponential decay, confirming that spatial proximity directly dictates structural influence. By truncating the Neumann series at $K=3$, both the uncompensated $B_K$ and compensated $\hat{B}_K$ highly approximate the exact values within the $K$-hop neighborhood ($d \le 3$). Critically, for distances $d(i,j) > K$, the entry values strictly drop to the numerical floor. This phenomenon empirically demonstrates the enforcement of a hard localized receptive field, which prevents the unbounded integration of distant topological noise and intrinsically limits the over-smoothing problem. Furthermore, the results validate the orthogonality of the Residual Mass Compensation mechanism. The compensated matrix $\hat{B}_K$ strictly diverges from the uncompensated matrix $B_K$ only at $d=0$ (the self-loop entries), where the residual mass successfully restores the local centrality indicator. For all distances $d(i,j) \ge 1$, the two approximations remain perfectly identical, proving that the compensation ensures exact row-stochasticity without violating the strict boundaries of the finite receptive field.

\subsection{Dirichlet Energy Decay and Over-smoothing Mitigation}
\label{subsec:dirichlet_energy_analysis}

\begin{figure*}[htbp]
    \centering
    \includegraphics[width=\textwidth]{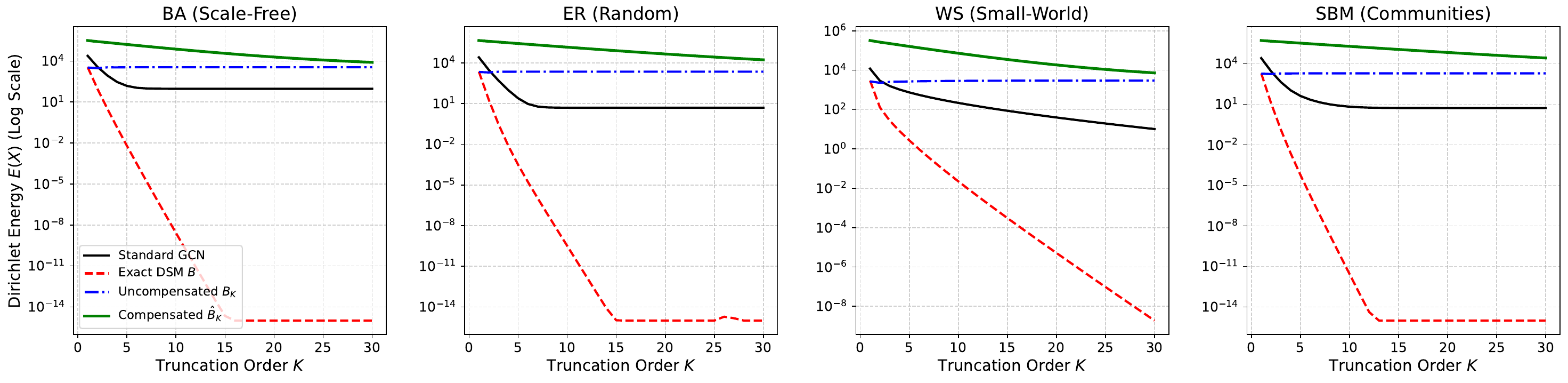}
    \caption{Dirichlet energy decay trajectories across four distinct graph topologies (BA, ER, WS, and SBM, $N=1000$). Starting from identical random Gaussian features, the exact doubly stochastic matrix $B$ (red dashed) forces an exponential collapse to numerical zero, indicating absolute over-smoothing. The uncompensated truncated series $B_K$ (blue dash-dotted) plateaus prematurely due to severe probability mass leakage, halting effective structural aggregation. The standard GCN (black solid) plateaus at a low non-zero energy bound. In contrast, our compensated operator $\hat{B}_K$ (green solid) consistently maintains the highest stationary energy lower bound across all topologies, empirically verifying that the Residual Mass Compensation mechanism successfully anchors high-frequency structural signals and averts catastrophic over-smoothing.}
    \label{fig:dirichlet_energy_4_graphs}
\end{figure*}

To rigorously evaluate the model's resilience against over-smoothing, we track the evolution of the Dirichlet energy $E(X) = \frac{1}{2} \text{Tr}(X^T L X)$ across four distinct graph topologies: Scale-Free (BA), Random (ER), Small-World (WS), and Community-Structured (SBM). The Dirichlet energy serves as a standard mathematical proxy for feature smoothness; a continuous decay towards zero implies that node representations are becoming indistinguishable. Figure~\ref{fig:dirichlet_energy_4_graphs} illustrates the energy trajectories starting from identical random Gaussian noise $X \sim \mathcal{N}(0, 1)$. The empirical results reveal consistent algebraic behaviors regardless of the underlying topology. The standard GCN operator (black solid line) exhibits continuous energy depletion, eventually plateauing at a very low bound. Conversely, recursive application of the exact doubly stochastic matrix $B$ (red dashed line) forces the energy to collapse exponentially to numerical zero ($<10^{-14}$), representing absolute homogenization. Furthermore, the uncompensated truncated operator $B_K$ (blue dash-dotted line) exhibits a severe algebraic collapse; its energy curve flattens almost immediately, indicating that the operator ceases to aggregate meaningful structural information beyond the local neighborhood due to probability mass leakage. In stark contrast, our compensated operator $\hat{B}_K$ (green solid line) completely mitigates these issues via the Residual Mass Compensation mechanism. By analytically re-injecting the truncated probability mass back into the ego-node, $\hat{B}_K$ guarantees a strictly high, non-zero stationary energy bound. This behavior remains exceptionally robust across all tested graph structures. The mathematically enforced local anchor strictly prevents the dilution of high-frequency structural signals, thereby allowing the architecture to capture structural dependencies without triggering catastrophic over-smoothing.

\section{Heterophilic Graphs Performance Analysis}
\label{app:heterophilic}

\begin{table*}[htbp]
\centering
\caption{Node classification accuracy (percentage) on heterophilic datasets under \textbf{Full Supervised (FULL)} setting. Standard deviations are also scaled. Legend: \colorbox{first}{\scriptsize 1st}, \colorbox{second}{\scriptsize 2nd}, \colorbox{third}{\scriptsize 3rd}.}
\label{tab:heterophilic_full}
\resizebox{\textwidth}{!}{
\begin{tabular}{lccccc}
\toprule
\textbf{Model} & \textbf{Chameleon} & \textbf{Squirrel} & \textbf{Texas} & \textbf{Wisconsin} & \textbf{Cornell} \\
\midrule
GCN       & 57.7 ± 1.6 & \cellcolor{third}35.4 ± 0.2 & 63.7 ± 0.5 & 62.0 ± 0.5 & 48.9 ± 0.8 \\
GAT       & \cellcolor{third}57.8 ± 0.2 & 32.4 ± 0.5 & 59.6 ± 0.4 & 60.0 ± 1.0 & 44.1 ± 1.2 \\
SAGE      & \cellcolor{second}59.7 ± 1.1 & \cellcolor{second}37.2 ± 0.1 & \cellcolor{first}82.3 ± 0.9 & \cellcolor{first}84.3 ± 0.4 & \cellcolor{first}72.3 ± 0.8 \\
CHEB      & 43.6 ± 0.5 & 30.8 ± 0.3 & \cellcolor{second}74.3 ± 0.5 & \cellcolor{second}77.3 ± 0.4 & \cellcolor{second}68.0 ± 0.7 \\
APPNP     & 48.5 ± 0.1 & 23.5 ± 0.3 & 61.1 ± 0.7 & 61.6 ± 0.5 & 41.9 ± 1.8 \\
GCNII     & 40.7 ± 0.6 & 29.1 ± 0.3 & 59.0 ± 0.7 & 55.9 ± 0.9 & 42.0 ± 0.4 \\
JKNET     & \cellcolor{first}59.9 ± 0.4 & \cellcolor{first}42.5 ± 0.5 & 59.0 ± 1.2 & 56.1 ± 0.6 & 42.1 ± 1.2 \\
\midrule
DSM       & 47.8 ± 0.3 & 25.4 ± 0.4 & 62.7 ± 0.8 & 64.9 ± 1.1 & 52.7 ± 1.7 \\
DSM\_COMP & 51.1 ± 0.3 & 31.4 ± 0.2 & \cellcolor{third}65.1 ± 1.1 & \cellcolor{third}68.3 ± 0.5 & \cellcolor{third}53.4 ± 2.0 \\
\bottomrule
\end{tabular}
}
\end{table*}

\begin{table*}[htbp]
\centering
\caption{Node classification accuracy (percentage) on heterophilic datasets under \textbf{Semi-Supervised (SEMI)} setting. Standard deviations are also scaled. Legend: \colorbox{first}{\scriptsize 1st}, \colorbox{second}{\scriptsize 2nd}, \colorbox{third}{\scriptsize 3rd}.}
\label{tab:heterophilic_semi}
\resizebox{\textwidth}{!}{
\begin{tabular}{lccccc}
\toprule
\textbf{Model} & \textbf{Chameleon} & \textbf{Squirrel} & \textbf{Texas} & \textbf{Wisconsin} & \textbf{Cornell} \\
\midrule
GCN       & \cellcolor{first}41.2 ± 4.5 & \cellcolor{third}24.3 ± 3.6 & 63.8 ± 4.7 & \cellcolor{third}52.8 ± 6.8 & 47.0 ± 4.0 \\
GAT       & \cellcolor{second}40.3 ± 4.5 & 23.1 ± 2.1 & 65.4 ± 5.5 & 48.4 ± 7.3 & 45.4 ± 7.5 \\
SAGE      & \cellcolor{third}40.2 ± 2.8 & 23.7 ± 1.3 & \cellcolor{first}80.0 ± 8.1 & \cellcolor{first}71.6 ± 5.7 & \cellcolor{first}66.5 ± 4.4 \\
CHEB      & 30.5 ± 3.2 & 23.6 ± 2.0 & \cellcolor{second}67.6 ± 5.7 & \cellcolor{second}62.0 ± 10.9 & \cellcolor{second}54.6 ± 6.7 \\
APPNP     & 31.8 ± 4.2 & 19.9 ± 0.9 & 31.9 ± 20.2 & 48.4 ± 7.7 & 47.6 ± 11.2 \\
GCNII     & 27.9 ± 3.0 & \cellcolor{second}25.0 ± 1.6 & 40.5 ± 24.5 & 46.0 ± 9.1 & 35.7 ± 13.4 \\
JKNET     & 34.9 ± 3.6 & \cellcolor{first}26.8 ± 1.7 & 48.7 ± 19.4 & 48.4 ± 9.9 & 42.2 ± 11.5 \\
\midrule
DSM       & 32.7 ± 4.9 & 22.0 ± 1.3 & \cellcolor{third}66.5 ± 7.6 & 30.4 ± 18.7 & 31.9 ± 18.7 \\
DSM\_COMP & 31.8 ± 5.1 & 23.9 ± 2.7 & 65.9 ± 8.5 & 39.2 ± 23.2 & \cellcolor{third}51.9 ± 4.7 \\
\bottomrule
\end{tabular}
}
\end{table*}

\section{Integration with Graph Transformers}
\label{app:graph_transformer_integration}

To further evaluate the versatility of the Doubly Stochastic Graph Matrix (DSM), we integrate it into generalized Graph Transformer architectures, specifically substituting their native structural and positional encodings. 

Traditional Graph Transformers rely on heuristic or discrete metrics to inject graph topology into the self-attention mechanism. For instance, Graphormer utilizes discrete shortest-path distances for spatial encoding and degree metrics for node centrality. Similarly, the standard Graph Transformer (GT) employs Laplacian eigenvectors to establish absolute positional encodings. While effective, these methods often struggle to capture continuous diffusion dynamics and are strictly bounded by discrete spatial steps or spectral approximations.

In our implementation, we systematically replace these mechanisms with the DSM. We extract the diagonal of the structural matrix $B$ (or its truncated counterparts) to serve as a continuous, diffusion-based \textit{Centrality Encoding}, directly reflecting a node's global or $K$-hop structural dominance. Concurrently, the dense off-diagonal elements act as \textit{Spatial Encodings} (or edge attributes), feeding precise, topology-aware transition probabilities into the attention bias. 

\begin{table*}[htbp]
\centering
\caption{Node classification accuracy (percentage) using Graph Transformer architectures. Results compare original baseline encodings against Exact, Truncated, and Compensated Truncated DSM encodings. Standard deviations are scaled accordingly. Legend: \colorbox{first}{\scriptsize 1st}, \colorbox{second}{\scriptsize 2nd}, \colorbox{third}{\scriptsize 3rd} per architecture.}
\label{tab:transformer_results}
\resizebox{\textwidth}{!}{
\begin{tabular}{llccc}
\toprule
\textbf{Base Architecture} & \textbf{Encoding Mode} & \textbf{Cora} & \textbf{CiteSeer} & \textbf{Photo} \\
\midrule
\multirow{4}{*}{\textbf{Graph Transformer (GT)}} 
& Original (Laplacian PE) & \cellcolor{third}82.9 ± 2.6 & \cellcolor{second}68.6 ± 1.7 & \cellcolor{third}94.4 ± 0.4 \\
& Exact DSM               & \cellcolor{second}83.2 ± 1.9 & 65.0 ± 4.9 & \cellcolor{first}94.7 ± 0.6 \\
& Truncated DSM           & 82.1 ± 2.6 & \cellcolor{third}65.5 ± 3.5 & 94.6 ± 0.8 \\
& Comp. Trunc. DSM        & \cellcolor{first}85.1 ± 1.0 & \cellcolor{first}69.7 ± 1.9 & \cellcolor{second}94.5 ± 0.5 \\
\midrule
\multirow{4}{*}{\textbf{Graphormer}}              
& Original (Degree + SPD) & 53.8 ± 7.3 & 52.8 ± 4.9 & \cellcolor{second}82.7 ± 3.6 \\
& Exact DSM               & \cellcolor{first}60.3 ± 1.6 & \cellcolor{third}57.9 ± 2.5 & \cellcolor{third}82.3 ± 2.7 \\
& Truncated DSM           & \cellcolor{second}58.5 ± 6.9 & \cellcolor{second}61.3 ± 4.9 & 85.3 ± 1.3 \\
& Comp. Trunc. DSM        & \cellcolor{third}57.8 ± 7.5 & \cellcolor{first}61.7 ± 4.7 & \cellcolor{first}84.9 ± 1.0 \\
\bottomrule
\end{tabular}
}
\end{table*}

This substitution provides several theoretical advantages for the self-attention mechanism:
\begin{itemize}
    \item \textbf{Continuous Proximity Measure:} Unlike discrete shortest-path distances, the DSM provides a continuous relative proximity score based on random walk probabilities. It naturally decays with distance while accounting for multiple parallel paths between node pairs, offering the attention mechanism a richer measure of topological relevance.
    \item \textbf{Robust Centrality:} The diagonal entries of the DSM measure the return probability of a random walk, providing a stricter and more robust centrality prior than simple degree counting, allowing the attention heads to accurately weigh structurally critical nodes.
    \item \textbf{Bounded Receptive Field with Mass Conservation:} By utilizing the Compensated Truncated DSM ($\hat{B}_K$), we enforce a hard $K$-hop spatial cutoff in the attention matrix while preserving topological mass. This explicitly prevents the self-attention mechanism from integrating unbounded global noise, thereby mitigating over-smoothing in deep Transformer layers.
\end{itemize}

Table \ref{tab:transformer_results} summarizes the node classification performance of the Graph Transformer (GT) and Graphormer when utilizing their original encodings versus the exact, truncated, and compensated truncated DSM matrices. The empirical results demonstrate that integrating the compensated truncated DSM generally yields superior or highly competitive performance across datasets, confirming that diffusion-based structural encodings enhance the representational capacity of self-attention on graphs.

\end{document}